%% file: 0_main_iccv.tex
\documentclass[10pt,twocolumn,letterpaper]{article}
\usepackage[accsupp]{axessibility}

\usepackage{iccv}
\usepackage{times}
\usepackage[T1]{fontenc}
\usepackage[utf8]{inputenc}
\usepackage{amsfonts}
\usepackage{amsmath}
\usepackage{amssymb}
\usepackage{amsthm}
\usepackage{mathtools}
\usepackage{bigdelim}
\usepackage{booktabs}
\usepackage{caption}
\usepackage{chngcntr}
\usepackage{enumitem}
\usepackage{graphicx}
\usepackage{xcolor}
\usepackage{color,colortbl}
\usepackage{microtype}
\usepackage{multirow}
\usepackage{nicefrac}
\usepackage{pifont}
\usepackage{rotating}
\usepackage{stackrel}
\usepackage{subcaption}
\usepackage{tabularx}
\usepackage{tikz}
\usepackage[colorinlistoftodos]{todonotes}
\usepackage{url}
\usepackage{wrapfig,lipsum,booktabs}

\usepackage[numbers,sort]{natbib}
\usepackage[pagebackref=true,breaklinks=true,letterpaper=true,colorlinks,bookmarks=false]{hyperref}
\usepackage{cleveref}

\usetikzlibrary{trees,fit,shapes,calc}

\input{vedaldi}

\newtheorem{theorem}{Theorem}
\newtheorem{lemma}{Lemma}
\newtheorem{definition}{Definition}

\crefname{section}{Sec.}{Sec.}
\Crefname{section}{Sec.}{Sec.}
\crefname{equation}{eq.}{eq.}
\Crefname{equation}{Eq.}{Eq.}
\crefname{appendix}{Appendix}{Appendix}
\Crefname{appendix}{Appendix}{Appendix}

\makeatletter
\DeclareRobustCommand{\cev}[1]{%
  \mathpalette\do@cev{#1}%
}
\newcommand{\do@cev}[2]{%
  \fix@cev{#1}{+}%
  \reflectbox{$\m@th#1\vec{\reflectbox{$\fix@cev{#1}{-}\m@th#1#2\fix@cev{#1}{+}$}}$}%
  \fix@cev{#1}{-}%
}
\newcommand{\fix@cev}[2]{%
  \ifx#1\displaystyle
    \mkern#23mu
  \else
    \ifx#1\textstyle
      \mkern#23mu
    \else
      \ifx#1\scriptstyle
        \mkern#22mu
      \else
        \mkern#22mu
      \fi
    \fi
  \fi
}
\newcommand{\Tone}{\colorbox{blue!30}{$T_1$}\xspace}
\newcommand{\Ttwo}{\colorbox{blue!20}{$T_2$}\xspace}
\newcommand{\Tthree}{\colorbox{blue!10}{$T_3$}\xspace}
\newcommand{\Tfour}{\colorbox{blue!5}{$T_4$}\xspace}
\newcommand{\Tfive}{\colorbox{red!30}{$T_5$}\xspace}
\newcommand{\Tsix}{\colorbox{red!20}{$T_6$}\xspace}
\newcommand{\Tseven}{\colorbox{red!10}{$T_7$}\xspace}
\newcommand{\Teight}{\colorbox{red!5}{$T_8$}\xspace}
\makeatother

\newcommand{\midsepremove}{\aboverulesep = 0.4mm \belowrulesep = 0.4mm}
\midsepremove
\newcommand{\midsepdefault}{\aboverulesep = 0.605mm \belowrulesep = 0.984mm}
\midsepdefault

\newcommand{\ul}[1]{\underline{{#1}}}
\newcommand\methodname{GDT}
\def\optrow#1\\{} 
\def\labelswitch#1{\label{#1}}
\def\dontshowthisinappendix#1{#1}
\def\showthisinappendix#1{}
\usepackage[figurename=Fig.]{caption}

\iccvfinalcopy 

\def\iccvPaperID{1009} 
\def\httilde{\mbox{\tt\raisebox{-.5ex}{\symbol{126}}}}

\ificcvfinal\fi

\makeatletter
\renewcommand{\paragraph}{%
  \@startsection{paragraph}{4}%
  {\z@}{0em}{-1em}%
  {\normalfont\normalsize\bfseries}%
}
\makeatother

\author{%
  Mandela Patrick$^{1,2}$\thanks{Joint first authors} \qquad
  Yuki M. Asano$^2$\footnotemark[1] \qquad Polina Kuznetsova$^1$ \qquad Ruth Fong$^2$ \\ \qquad João F. Henriques$^2$
  \qquad Geoffrey Zweig$^1$ \qquad \qquad \qquad Andrea Vedaldi$^{1,2}$ \vspace{1mm}
\\
  $^1$ Facebook AI Research \\  \texttt{mandelapatrick@facebook.com}\vspace{1mm}\\
  $^2$ Visual Geometry Group, University of Oxford \\ \texttt{yuki@robots.ox.ac.uk} \\
}

\title{On Compositions of Transformations in \\Contrastive Self-Supervised Learning}

\begin{document}

\maketitle
\def\httilde{\mbox{\tt\raisebox{-.5ex}{\symbol{126}}}}
\ificcvfinal\pagestyle{empty}\fi



\begin{abstract}
In the image domain, excellent representations can be learned by inducing invariance to content-preserving transformations via noise contrastive learning.
In this paper, we generalize contrastive learning to a wider set of transformations, and their compositions, for which either invariance \textit{or} distinctiveness is sought.
We show that it is not immediately obvious how existing methods such as SimCLR can be extended to do so.
Instead, we introduce a number of formal requirements that all contrastive formulations must satisfy, and propose a practical construction which satisfies these requirements.
In order to maximise the reach of this analysis, we express all components of noise contrastive formulations as the choice of certain generalized transformations of the data (GDTs), including data sampling.
We then consider videos as an example of data in which a large variety of transformations are applicable, accounting for the extra modalities -- for which we analyze audio and text -- and the dimension of time.
We find that being invariant to certain transformations and distinctive to others is critical to learning effective video representations, improving the state-of-the-art for multiple benchmarks by a large margin, and even surpassing supervised pretraining.
Code and pretrained models are available\footnote{https://github.com/facebookresearch/GDT}.

\end{abstract}


\input{1_introduction}

\input{2_related_work}
\input{3_methods}
\input{4_experiments}

\input{5_conclusion}


\paragraph{Acknowledgements.}
We are grateful for support from the Rhodes Trust (M.P.), Facebook (M.P.), EPSRC Centre for Doctoral Training in Autonomous Intelligent Machines \& Systems [EP/L015897/1] (M.P. and Y.A.), the Qualcomm Fellowship (Y.A.), the Open Philanthropy Project (R.F.), and the Royal Academy of Engineering under the Research Fellowship scheme (J.F.H.).
We also thank Andrew Owens, Alexei Efros, Triantafyllos Afouras, Gül Varol and Tengda Han for helpful discussions and feedback.
Lastly, the authors would also like to thank Ishan Misra and Bruno Korbar from Facebook for valuable discussions and sharing insights.

{\bibliographystyle{ieee_fullname}\bibliography{shortstrings,refs,vedaldi_general,vedaldi_specific}}

\def\optrow#1\\{#1\\} 
\def\labelswitch#1{\label{supp:#1}}
\def\dontshowthisinappendix#1{}
\def\showthisinappendix#1{#1}

\clearpage
\input{6_appendix}

\end{document}

%% file: vedaldi.tex
\newif\ifdark
\IfFileExists{/Users/vedaldi/.bash_profile}{
\immediate\write18{%
if defaults read -g AppleInterfaceStyle 2>/dev/null;
then echo \\darktrue > /tmp/displaymode.tex;
else echo \\darkfalse > /tmp/displaymode.tex; fi}
\input{/tmp/displaymode.tex}
\ifdark
\definecolor{pcolor}{HTML}{1E1E1E}
\definecolor{tcolor}{HTML}{C5C5C5}
\else
\definecolor{pcolor}{HTML}{FDF6E3}
\definecolor{tcolor}{HTML}{333333}
\fi
\pagecolor{pcolor}
\color{tcolor}
\hbadness=\maxdimen
\vbadness=\maxdimen
\vfuzz=30pt
\hfuzz=30pt
}{}

%% file: 1_introduction.tex
\section{Introduction}

Works such as MoCo~\citep{he2019momentum}, SimCLR~\citep{Chen2020ASF}, SwAV~\citep{caron2020unsupervised} and BYOL~\citep{grill2020bootstrap} have shown that it is possible to pre-train state-of-the-art image representations without the use of any manually-provided labels.
Furthermore, many of these approaches use variants of noise contrastive learning~\citep{gutmann2010noise, Hadsell2006DimensionalityRB}.
Their idea is to learn a representation that is \emph{invariant} to transformations that leave the meaning of an image unchanged (e.g.~geometric distortion or cropping) and \emph{distinctive} to changes that are likely to alter its meaning (e.g.~replacing an image with another chosen at random).

\input{figs/splash_update}

These prior works have also shown that the choice of transformations is of primary importance for performance~\citep{Chen2020ASF, caron2020unsupervised}.
This is not just a matter of selecting a certain type of transformation, but also to specify how different transformations should be composed, and how these compositions should be sampled to from batches for training the model.
So far, these choices have been mostly driven by intuition, with little formal understanding of why certain choices may be preferable, and how these choices can be generalized.

In this work, we answer some of these questions via a formal analysis of \emph{composable transformations in contrastive learning}.
Our analysis shows how invariance and distinctiveness to individual transformations can be meaningfully incorporated in the same learning formulation.
It also provides some principles to guide the construction of the training batches.
We interpret existing sampling schemes, such as the one in SimCLR, as special cases with certain potential advantages and disadvantages.
We do so by showing how these constructions can be extended systematically to any composition of invariant and distinctive transformations.

Furthermore, we demonstrate the utility of our analysis by exploring contrastive methods for learning representations of video data.
Compared to images, videos contain a time dimension and multiple modalities, which have been shown to provide effective learning cues; for instance
\citep{owens2016ambient} leverages multiple modalities, and \citep{avts,Chung16a} the time dimension.
We show how these effects can be incorporated in a uniform manner in contrastive learning by considering a suitable class of \emph{generalized data 
transformations} (GDTs).
GDTs capture standard augmentations, as well as temporal transformations, modality slicing and data sampling.
The advantages of using GDTs is that they allow us to base the entire design of the learning formulation (e.g., how to write a coherent learning objective and how to sample batches) on a small number of design principles that our analysis has identified.

With this, we make some notable findings for contrastive video representation learning.
First, we show that using this wider class of transformations greatly exceeds the performance that can be obtained by a vanilla applications of image-centric methods such as SimCLR to video data.
By leveraging time and multiple modalities, we obtain large performance gains, almost \emph{doubling} the performance.
Second, we show that just learning representations that are invariant to more and more transformations is \emph{not} optimal, at least when it comes to video data; instead, combining invariance to certain factors with distinctiveness to others performs better.
To the best of our knowledge, this is the first time such an effect has been demonstrated in contrastive learning.


We also set the new state of the art in audio-visual representation learning, with both  small and large video pretraining datasets on a variety of downstream tasks.
In particular, we achieve $94.1\%$ and $67.4\%$ on the standardized UCF-101~\citep{UCF101} and HMDB-51~\citep{kuehne2011hmdb} action recognition benchmarks, when pretrained on HowTo100M~\citep{miech2019howto100m}, and $95.2\%$ and $72.8\%$ respectively when pretrained on IG65M~\citep{Ghadiyaram2019}.

%% file: figs/splash_update.tex
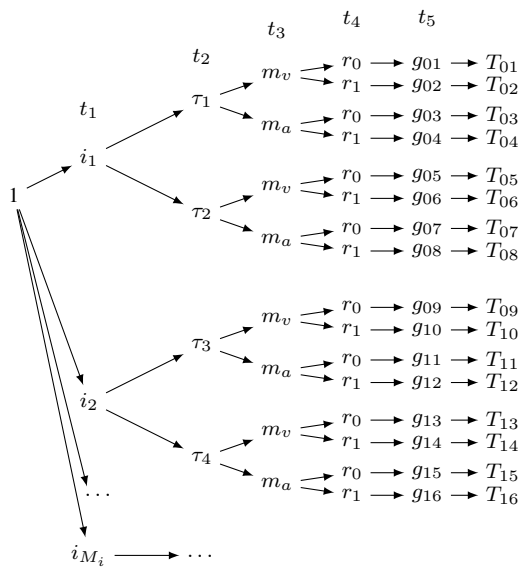
\begin{figure}

\renewcommand{\Tone}{{$T_{01}$}\xspace}
\renewcommand{\Ttwo}{{$T_{02}$}\xspace}
\renewcommand{\Tthree}{{$T_{03}$}\xspace}
\renewcommand{\Tfour}{{$T_{04}$}\xspace}
\renewcommand{\Tfive}{{$T_{05}$}\xspace}
\renewcommand{\Tsix}{{$T_{06}$}\xspace}
\renewcommand{\Tseven}{{$T_{07}$}\xspace}
\renewcommand{\Teight}{{$T_{08}$}\xspace}
\newcommand{\Tnine}{{$T_{09}$}\xspace}
\newcommand{\Tten}{{$T_{10}$}\xspace}
\newcommand{\Televen}{{$T_{11}$}\xspace}
\newcommand{\Ttwelve}{{$T_{12}$}\xspace}
\newcommand{\Tthirteen}{{$T_{13}$}\xspace}
\newcommand{\Tfourteen}{{$T_{14}$}\xspace}
\newcommand{\Tfifteen}{{$T_{15}$}\xspace}
\newcommand{\Tsixteen}{{$T_{16}$}\xspace}

\fboxsep0.5pt
\footnotesize
\begin{tikzpicture}[grow'=right,
  edge from parent/.style={draw,-latex},
  node distance=10mm,
  label distance = 0.2cm,
  level 3/.style = {level distance=1cm, sibling distance=7mm},
  level 4/.style = {level distance=1cm, sibling distance=3mm},
  ]
\node (X) {1}
child { node[label={$t_1$}] at ($(X)+(1,+0.5)$) {$i_1$}
    child { node[label={$t_2$}] {$\tau_1$} 
        child { node[label={$t_3$}] {$m_v$}
            child { node[label={$t_4$}]{$r_0$} 
                 child { node[label={$t_5$}]{$g_{01}$} child { node{\Tone}} } }
            child { node{$r_1$} 
                 child { node{$g_{02}$} child { node{\Ttwo}} } }
        }
        child { node {$m_a$}
            child { node {$r_0$} child { node{$g_{03}$} child { node{\Tthree}} } }
            child { node {$r_1$} child { node{$g_{04}$} child { node{\Tfour}} } }
        }
    }
    child { node  {$\tau_2$}
        child { node {$m_v$}
            child { node{$r_0$} child { node{$g_{05}$} child { node{\Tfive}} } }
            child { node{$r_1$} child { node{$g_{06}$} child { node{\Tsix}} } }
        }
        child { node {$m_a$}
            child { node{$r_0$} child { node{$g_{07}$} child { node{\Tseven}} } }
            child { node{$r_1$} child { node{$g_{08}$} child { node{\Teight}} } }
        }
    }
}
child { node at ($(X)+(1,-2.75)$) {$i_2$}
    child { node {$\tau_3$}
        child { node {$m_v$}
            child { node{$r_0$} child { node{$g_{09}$} child { node{\Tnine}} } }
            child { node{$r_1$} child { node{$g_{10}$} child { node{\Tten}} } }
        }
        child { node {$m_a$}
            child { node{$r_0$} child { node{$g_{11}$} child { node{\Televen}} } }
            child { node{$r_1$} child { node{$g_{12}$} child { node{\Ttwelve}} } }
        }
    }
    child { node{$\tau_4$}
        child { node {$m_v$} 
            child { node{$r_0$} child { node{$g_{13}$} child { node{\Tthirteen}} } }
            child { node{$r_1$} child { node{$g_{14}$} child { node{\Tfourteen}} } }
        }
        child { node {$m_a$}
            child { node{$r_0$} child { node{$g_{15}$} child { node{\Tfifteen}} } }
            child { node{$r_1$} child { node{$g_{16}$} child { node{\Tsixteen}} } }
        }
    }
}
child { node at ($(X)+(1,-4.0)$) {$\quad \dots$} }
child { node at ($(X)+(1,-4.8)$) {$i_{M_i}$}
    child { node{$\dots$} }
};
\end{tikzpicture}
\hfill
\caption{
    Hierarchical sampling process of generalized data transformations (GDTs).
    Shown here are the five transformations analyzed for the audio-visual training case and their compositions: data-sampling ($t_1$), time-shift  ($t_2$), modality splicing ($t_3$), time-reversal ($t_4$), and augmentation transformations, $g$ ($t_5$) to learn video representations via noise contrastive learning.
}\label{f:theory}
\end{figure}

%% file: 2_related_work.tex
\section{Related work}


\textbf{Self-supervised learning from images and videos.}
A variety of pretext tasks have been proposed to learn representations from unlabelled \textbf{images}.
Some tasks leverage the spatial context in images \citep{doersch2015unsupervised, noroozi2016unsupervised} to train CNNs, while others create pseudo classification labels via artificial rotations \citep{gidaris2018unsupervised}, or clustering features \citep{ji2018invariant,caron2018deep, caron2019unsupervised,caron2020unsupervised,asano2020self,gidaris2020learning}.
Colorization~\citep{zhang16colorful,zhang2017split},
inpainting~\citep{pathak2016context}, solving jigsaw puzzles~\citep{noroozi2017representation}, as well as the contrastive methods detailed below, have been proposed for self-supervised image representation learning.
Some of the tasks that use the space dimension of images have been extended to the space-time dimensions of \textbf{videos} by crafting equivalent tasks.
These include jigsaw puzzles \citep{kim2019self}, and predicting rotations~\citep{jing2018self} or future frames~\citep{han2019video}.
Other tasks leverage the temporal dimension of videos to learn representations by predicting shuffled frames~\citep{misra2016shuffle}, the direction of time~\citep{wei2018learning}, motion \citep{motion_statistics}, temporal ordering~\citep{clip_order, lee2017unsupervised}, and playback speed \citep{benaim2020speednet, cho2020selfsupervised,fernando17self-supervised}.
These pretext-tasks can be framed as GDTs.

\textbf{Multi-modal learning.}
Videos, unlike images, are a rich source of a variety of modalities such as speech, audio, and optical flow, and their correlation can be used as a supervisory signal.
This idea has been present as early as 1994~\citep{NIPS1993_831}.
Only recently, however, has multi-modal learning been used to successfully learn effective representations by leveraging the natural correspondence~\citep{owens2016ambient, aytar2016soundnet, Arandjelovic17, morgado2020avid, alwassel2019self, asano2020labelling} and synchronization~\citep{Chung16a, avts, owens2018audio} between the audio and visual streams.
A number of recent papers have leveraged speech as a weak supervisory signal to train video representations \citep{sun2019videobert, sun2019contrastive, miech2019endtoend, li2020learning, nagrani2020} and recently \cite{alayrac2020selfsupervised}, who use speech, audio and video.
Other works incorporate optical flow and other modalities~\citep{zhao2019sound, piergiovanni2020evolving, Han2020SelfsupervisedCF, Han_2020} to learn representations.
In CMC~\citep{tian2019contrastive}, representations are learned with different views such as different color channels or modalities to solely induce multi-view invariance.
In contrast, our work extends this to and analyses multi-modal transformations and examines their utility as an invariant \textit{or} distinctive learning signal.

\textbf{Noise Contrastive Loss.}
Noise contrastive losses~\citep{Hadsell2006DimensionalityRB, gutmann2010noise} measure the similarity between sample pairs in a representational space and are at the core of several recent works on unsupervised feature learning.
They yield good performance for learning image~\citep{Wu_2018_CVPR, oord2018representation,hjelm2018learning, tian2019contrastive, he2019momentum, misra2019selfsupervised,henaff2019data,Chen2020ASF, li2020prototypical,tian2020makes} and video~\citep{sohn2016improved, miech2019endtoend, han2019video, sun2019contrastive, li2020learning, morgado2020avid, yao2021seco, anand2019unsupervised, hjelm2020representation} representations, and circumvent the need to explicitly specify what information needs to be discarded via a designed task.

We leverage the noise contrastive loss as a learning framework to encourage the network to learn desired invariance and distinctiveness to data transformations.
The GDT framework can be used to combine and extend many of these cues, contrastive or not, in a single noise contrastive formulation.



%% file: 3_methods.tex
\input{figs/splash}

\section{Method}\label{s:method-new}

We generalize contrastive methods such as CPC~\citep{oord2018representation}, PIRL~\citep{misra2019selfsupervised}, MoCo~\citep{he2019momentum}, SimCLR~\citep{Chen2020ASF}, and SwAV~\citep{caron2020unsupervised} to learn representations that can be invariant or distinctive to any number of transformations.

Given a collection $x$ of data such as images or videos, we generate training samples
$$
   x(t_1,\dots,t_M) \in \mathcal{X}.
$$
by applying a sequence of $M$ transformations \mbox{$T=(t_1,\dots,t_M)$} to the collection.
We consider typical transformations such as data augmentations (e.g., randomly cropping an image).
We also find it useful to express in the same manner other operations such as extracting a specific image or video from the collection or extracting a specific modality from a video.
We call these \emph{generalized data transformations} (GDTs).

To provide a concrete example, in a standard contrastive learning formulation such as SimCLR, the first transformation $t_1=i\in\{1,\dots,|x|\}$ extracts an image $x_i$ from the collection $x$ and the second transformation $t_2=g$ applies to it a random augmentation, so that we can write
$
   x(t_1,t_2) = g(x_i).
$
The goal is to learn a representation $\Phi : \mathcal{X}\rightarrow \mathbb{R}^d$ that identifies the image regardless of the augmentation;
in other words, $\Phi$ should be invariant to the choice of $t_2$ and distinctive for the choice of $t_1$.

We wish to generalize this construction to richer data such as videos.
Compared to images, videos contain multiple modalities and additional dimensions, which allows to consider qualitatively different transformations such as time shift, time reversal, and modality slicing.
This generalization is however non-trivial.
First, when considering $M>2$ transformations, we have a choice of making the representation invariant or distinctive to each of them independently.
For instance, video representations may benefit from being distinctive to time shift and/or time reversal rather than invariant to them.
It is not immediately obvious how contrastive learning should be modified to incorporate these different choices.
Another less apparent but important issue is how training data batches should be formed.
Contrastive learning formulations minimize, in fact, a loss that involves comparing (contrasting) the representations of different samples, and is thus \emph{not decomposable}.
In practice, the loss is approximated by sampling batches of data, and how this is done has a major effect on the performance.
In the previous example of SimCLR, if transformation $(t_1,t_2)$ is included in the batch, so must be a complementary transformation $(t_1,t_2')$ that differs only in the second factor $t_2\not=t'_2$.
This is required in order to learn the desired invariance.
It also means that transformations in a batch cannot be sampled independently.
A way to guarantee that both $(t_1,t_2)$ and $(t_1,t_2')$ are in the batch is to consider all possible combinations
$
\mathcal{T}_1\times\mathcal{T}_2
$
of two sets of transformations $\mathcal{T}_1$ and $\mathcal{T}_2$.
However this is statistically inefficient because it applies the same augmentations $\mathcal{T}_2$  to all images in the batch.
Instead, SimCLR samples at random $B/2$ images and then applies to them $B$ independently sampled augmentations.
This is better than the scheme above that would only use $B / |\mathcal{T}_1| = 2$ different augmentations.
However, it is unclear how this strategy for sampling diverse transformations can be extended to $M\! >\! 2$ factors.
This is studied next. 

\subsection{Compositional contrastive learning}\label{s:main-method}

Given a batch $\mathcal{T}$ of data transformations, we consider the learning objective:
\begin{multline}\label{e:contrastive}
\footnotesize
\mathcal{L}(\Phi; \mathcal{T})
=
-
\sum_{\mathclap{T,T' \in \mathcal{T}}}
c(T,T')w(T,T') 
\\
\cdot \log
\left(
\frac{
e^{\langle \Phi(x(T)), \Phi(x(T')) \rangle/\rho}
}{
\sum\limits_{\mathrlap{T'' \in \mathcal{T}}}
w(T,T'')~
e^{\langle \Phi(x(T)), \Phi(x(T'')) \rangle/\rho}
}
\right).
\end{multline}
where $\rho > 0$ is a temperature parameter.
The \emph{contrast function} $c(T,T')\in\{0,1\}$ has the following interpretation:
when $c(T,T')=1$, then the representations $\Phi(x(T))$ and $\Phi(x(T'))$ are pulled together (invariance), and when $c(T,T')=0$ they are pushed apart (distinctiveness).
For example, in SimCLR, we set $c(T,T') = c((i,g),(i',g')) = \delta_{i=i'}$ to push apart the representations of different images $(i,i')$ while remaining invariant to a transformation pair $(g,g')$.
The \emph{weight function} $w$ is a second binary function that focuses learning on more informative transformation pairs; for instance, SimCLR sets $w(T,T')=\delta_{T\not=T'}$ to avoid focusing learning invariance to identical transformation $T=T'$ as this is trivially satisfied.
Next, we provide a semi-formal analysis of this formulation, leaving the details to~\cref{s:theory}.

\paragraph{Multiple invariances and distinctiveness.}

The key to extending~\cref{e:contrastive} to $M>2$ transformation is to build the function $c(T,T')$.
We do this one factor a time.
If we wish the representation to be distinctive to factor $t_m$, we set $c(t_m,t_m')=\delta_{t_m=t_m'}$.
If we wish it to be invariant to it, we set $c(t_m,t_m')=1$.
In \cref{l:prod,l:prod2} (\cref{s:theory}), we show that, given these choices, the only consistent definition of $c(T,T')$ is the product $\prod_{m=1}^M c(t_m,t_m')$.
The intuition is as follows:
The representation $\Phi$ should distinguish samples $x(T)$ and $x(T')$ if, and only if, at least one of the distinctive factors in $T$ and $T'$ differs.

\paragraph{Forming a batch.}

Given $c$, the remaining challenge is to sample training batches $\mathcal{T}$ appropriately.
We start by deriving some requirements for $\mathcal{T}$ and then develop a sampling scheme that satisfies them (none of these are guaranteed by sampling $T$ and $T'$ independently).
(i) First, in order for~\cref{e:contrastive} not to be identically zero, $c(T,T')$ should be non-zero for at least \emph{some} choices of $T$ and $T'$ in the batch.
(ii) Furthermore, when $c(T,T')=1$, this should not be for the trivial case $T=T'$ (the one that SimCLR discounts by setting $w(T,T')=0$).
Based on the discussion above, the condition $c(T,T')=1~\wedge~T\not=T'$ means that all distinctive factors in $T$ and $T'$ agree and that at least an invariant factor differs. 
(iii) Additionally, for the fraction in~\cref{e:contrastive} not to be a constant, if $c(T,T')=1$, there should be another $T''$ in the batch such that $c(T,T'')=0$.
The latter means that at least one distinctive factor in $T$ and $T''$ differs.

Short of considering all possible combinations of transformations (which, as explained 
above, can be statistically inefficient), we can sample a batch $\mathcal{T}$ that satisfies these constraints as follows.
We describe this process for the case of $M=3$ transformations, but note that it extends immediately to any $M$ (this is done in~\cref{s:composition,s:batches}).
First, we sample $K_1$ versions of the first distinctive transformations $t_1$.
Then, for each $t_1$, we sample $K_2$ transformations $t_2$, also distinctive.
Finally, for each choice of $(t_1,t_2)$, we sample $K_3$ invariant transformations $t_3$.\footnote{Note that the sampling ordering is arbitrary; in particular, it needs not to be the same as the ordering in which transformations are applied to the data.}
We thus obtain a batch of $|\mathcal{T}|=K_1 K_2 K_3$ transformations.

This scheme has several desirable properties.
First, for every $T=(t_1,t_2,t_3)$, there is another $T'=(t_1,t_2,t_3')$ that agrees on the distinctive factors and differs in the invariant one (properties (i) and (ii)).
Second, there is a $T''=(t_1,t_2',t_3')$ or $T''=(t_1',t_2',t_3')$ that differs in one or more distinctive factors (property (iii)).
Third, the construction is balanced, in the sense that the number of transformations that share a particular factor value $t_m$ is the same for all values of $t_m$ (this number is $|\mathcal{T}|/(K_1\cdots K_m)$).
Furthermore, SimCLR is obtained as a special case.
Please see~\cref{l:count1,l:count2} (\cref{s:theory}) for an in-depth discussion.

\paragraph{Limitations.}

Despite the benefits, this scheme has also some limitations.
The main issue is that a difference in factor $t_m$ generally implies a difference in all subsequent factors as well, meaning that the representation may be unable to observe and thus learn to discriminate changes in all individual factors.
In~\cref{s:limitations}, we show why this is unlikely to be an issue for the practical cases considered here and in the literature.
However, we also suggest other practical cases where this \emph{can} be a significant issue, affecting even methods such as SimCLR\@.

\subsection{Properties of Generalized Data Transformations\label{s:theorems}}

In this section, we show that GDT's batch sampling strategy is statistically more efficient than naively-sampled pairs for contrastive learning. 
We do this by showing that GDT's objective has the same mean but a lower variance than sampling batches with eq. \ref{e:contrastive} directly, which would either enumerate all possible pairs of transformations (which is prohibitively expensive) or subsample it by sampling transformations independently. 
We assume that the distinctive transformations are injective. 
This must be approximately true, otherwise it is impossible for any method to be distinctive to such transformations. 
In fact, we can prove the following result:
\begin{theorem}
\label{thm:variance}Given a set of transformations $\mathcal{T}$, of which the distinctive transformations are injective, GDT is an unbiased estimate $\hat{\mathcal{L}}$ of the generalized contrastive loss (eq. \ref{e:contrastive}), i.e. $\mathbb{E}[\hat{\mathcal{L}}]=\mathcal{L}(\Phi;\mathcal{T})$. Furthermore, consider a batch of sampled compositions of $M$ transformations, with size $\prod_{j}^{M}K_{j}$, where $K_{m}$ is the number of samples for the $m$th transformation. Define $K_{I}=\prod_{j\in I}K_{j}$ and $K_{V}=\prod_{j\in V}K_{j}$, where $I$ and $V$ are the subsets of indices corresponding to invariant and distinctive transformations, respectively. Denote by $\mathcal{L}_{jj'}$ and $\sigma_{jj'}^{2}$ the mean and variance of the partial sum of the objective (eq. \ref{e:contrastive}) on the set $\mathcal{X}_{j}\times\mathcal{X}_{j'}$, with $\mathcal{X}_{j}=\{x(T^{I},T_{j}^{V}):T^{I}\in\mathcal{T}_{I}\}$, i.e. the sample pairs corresponding to distinctive transformations with indices $j$ and $j'$. Then, the variance of the GDT estimate is
\[
\mathbb{V}[\hat{\mathcal{L}}]=\frac{1}{K_{V}^{4}K_{I}^{2}}\sum_{jj'}^{K_{V},K_{V}}\sigma_{jj'}^{2}.
\]
The naive estimate's variance, on the other hand, is
\[
\mathbb{V}[\hat{\mathcal{L}}_{\textrm{d}}]=\frac{1}{K_{V}^{2}K_{I}^{2}}\sum_{jj'}^{K_{V},K_{V}}\sigma_{jj'}^{2}+\frac{1}{K_{V}^{2}}\sum_{jj'}^{K_{V},K_{V}}\left(\mathcal{L}_{jj'}-\mathcal{L}\right)^{2},
\]
which is larger by a multiplicative factor of $K_{V}^{2}$ and a further additive factor.
\, Proof: {\normalfont See Appendix A. \qed}  
\end{theorem}

This states that sampling data with GDT yields reduced variance, resulting in higher-quality gradients for learning the \emph{same} objective (since the estimate is unbiased), which is reflected empirically in our strong performance on numerous datasets and benchmarks. 
We note that this may apply to other methods built on the same sampling strategy but which compose transformations in different ways than GDT, as long as the requirements (i-iii) for forming a batch (sec. \ref{s:main-method}) are satisfied.

\subsection{Application to video data}
As a concrete instantiation of our framework, we consider video data and transformations of the type $T=(t_1,t_2,t_3,t_4,t_5) = (i,\tau,m,r,g)$, as shown in \cref{f:theory}, as follows.
The first component $i$ \textbf{selects} a video in the dataset.
We sample $K_i \gg 2$ indices/videos and assume distinctiveness, so that~$c(i,i') = \delta_{i=i'}$.
The second component $\tau$ contrasts different \textbf{temporal shifts}.
We sample $K_\tau=2$ different values of a delay $\tau$ uniformly at random, extracting a 1s clip $x_{i\tau}$ starting at time $\tau$.
For this contrast, we will test the distinctiveness and invariance hypotheses, as \cite{avts} indicate that the former may be preferable.
The third component $m$ contrasts \textbf{modalities}, projecting the video $x_{i\tau}$ to either its visual or audio component $m(x_{i\tau})$.
We assume invariance $c(m,m')=1$ and always sample two such transformations $m_v$ and $m_a$ to extract both modalities, so $K_m = 2$.
The fourth component $r$ contrasts \textbf{time reversal}~\citep{pickup2014seeing, wei2018learning}, which has not previously been explored in a contrastive or cross-modal setting.
This is given by a transformation $r\in\mathcal{R}=\{r_0, r_1\}$, where $r_0$ is the identity and $r_1$ flips the time dimension of its input tensor, so $K_r = 2$.
The final component $g$ applies a spatial and aural \textbf{augmentation} $x(T) = g(r(m(x_{i\tau})))$, also normalizing the data.
We assume invariance $c(g,g')=1$ and pick $K_g=1$, i.e.~augment each datum at this level in the sampling hierarchy.
These choices lead to  $K=K_i K_\tau K_m K_r K_g = 8K_i$ transformations $T$ in the batch $\mathcal{T}$ (in ablations, we also test a subset of these choices).

While we focus on modality splitting, time reversal and shift, note that we could use any transformation that can yield a useful learning signal, such as speed~\citep{benaim2020speednet, cho2020selfsupervised, wang2020self, jenni2020video, yang2020video, epstein2020oops} and temporal ordering~\citep{misra2016shuffle, clip_order, lee2017unsupervised, fernando2017self}.

\paragraph{Modality splitting.}

The modality splitting transformation $m$ is useful to capture correlation between modalities~\citep{avts,Arandjelovic17,owens2016ambient,aytar2016soundnet,wei2018learning}.
Modality splitting means that the nature of the sample $x(i,\tau,m,r,g)$ is either a sequence of frames ($m=m_v$) or a sound ($m=m_a$).
Formally, this means that $x(i,\tau,m,r,g)$ is an element of the direct sum $\mathcal{X}_v \oplus \mathcal{X}_a$ of visual and audio signals; likewise $g$, $r$ and $\Phi$ are defined on this direct sum.
In practice, this means that the transformation $g$ comprises a pair of augmentations $(g_v,g_a)$, where $g_v(v)$ extracts a fixed-size tensor by resizing to a fixed resolution of a random spatial crop of the input video $v$, and $g_a(a)$ extracts a spectrogram representation of the audio signal followed by SpecAugment~\citep{Park_2019} with frequency and time masking.
Likewise, $\Phi=(\Phi_v,\Phi_a)$ comprises a pair of neural networks, one for each modality, both valued in $\mathbb{R}^d$ (refer to~\Cref{s:appx:traindetails} for architectural details).
In the \cref{s:appx:cross-vs-within-modal}, we show that modality splitting is key for performance; thus, we extend SimCLR weight function $w$ to focus learning on only cross-modal pairs:
$w(T,T)=\delta_{i\not= i} \cdot \delta_{m\not=m'}$.

\subsection{Discussion: utility of GDT}

With our framework, we can now generalize current state of the art contrastive learning approaches such as SimCLR in a systematic and practical manner.
The theory above and in~\cref{s:theory} tells us what is the meaning of composing transformations, how a batch should be sampled and why, how this can be achieved by using a hierarchical sampling scheme that extends SimCLR, and what are the limitations of doing so.
A particular benefit is to allow to specify individually, for each transformation, if invariance or distinctiveness is sought,
whereas previous works lack this distinction and largely considered learning only invariances (SimCLR~\citep{Chen2020ASF}, AVID~\citep{morgado2020avid}), or distinctiveness (AoT~\citep{wei2018learning}) to all factors.
This property allows the flexible utilization of dataset specific transformations in the case of prior knowledge, or, as we have shown in this study, the exploration of useful signals by enumeration. Finding the best transformation signals can even be further optimized by methods such as Bayesian optimization.
Finally, compared to a direct application of previous state-of-the-art methods image-based methods such as SimCLR~\citep{Chen2020ASF}, PIRL~\cite{misra2019selfsupervised}, and MoCo~\citep{he2019momentum}, we can also seamlessly incorporate important cues such as cross-modal correlation, greatly improving downstream performance (see \cref{tab:appx:ablation:modality}).





%% file: figs/splash.tex
\begin{figure}[!t]
    \centering
    \includegraphics[width=0.7\linewidth]{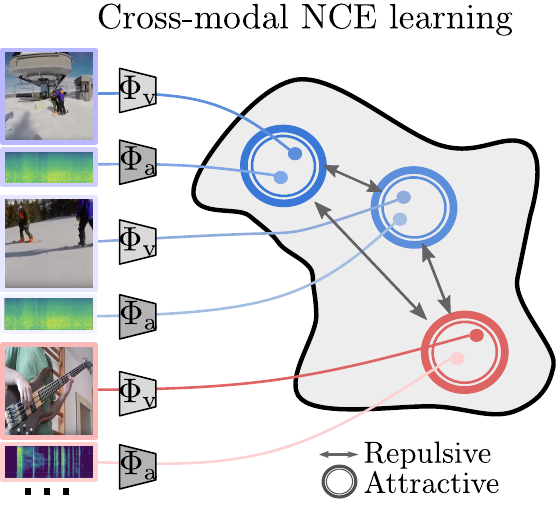}
    \caption{\textbf{Example instantiation.}
    The embedding is learned via noise contrastive (NCE) learning. Here we show the case of audio-visual sample and time-shift distinctiveness: video-audio embeddings from the same video at the same time are pulled together, while audio-visual sample pairs from different videos and different starting times are pushed apart. 
    \label{fig:nce}}
\end{figure}

%% file: 4_experiments.tex
\section{Experiments}\label{sec:experiments}

We compare self-supervised methods on pretraining audio-visual representations.
Quality is assessed based on how well the pretrained representation transfers to downstream tasks.
We conduct a study on video-audio, as well as video-text unsupervised representation learning to show the generality of our framework and then compare our best setup to the state of the art.

\paragraph{Self-supervised pretraining.}

For pretraining, we consider two standard pretraining datasets: Kinetics-400 ~\citep{kinetics} and HT100M~\citep{miech2019howto100m} and use R(2+1)D-18~\citep{tran2018closer} and a 2D ResNet~\citep{KaimingHe16} as encoders (see Appendix for further details). 
We also explore how our algorithm scales to even larger, less-curated datasets and train on IG65M~\citep{Ghadiyaram2019} as done in XDC~\citep{alwassel2019self}.


\paragraph{Downstream tasks.}
To assess the pretrained representation $f_v$, we consider standard action recognition benchmark datasets, UCF-101~\citep{UCF101} and HMDB-51~\citep{kuehne2011hmdb}.
We test the performance of our pretrained models on the tasks of finetuning the pretrained representation, conducting few-shot learning and video action retrieval.
The full details are given in the Appendix.

\subsection{Analysis of generalized data transformations}\label{s:exp.hypothesis}

\input{tabs/t2_hypothesis-res}
In this section, we conduct an extensive study on each parameter of the GDT transformation studied here, $T = (i, \tau, m, r, g)$, and evaluate the performance by finetuning our network and conduncting video retrieval on the HMDB-51 action recognition benchmark.

\paragraph{SimCLR-like baseline.}

First, we use the framework to test a direct extension of SimCLR to video data, as shown in \Cref{tab:hypotheses_results}(a)-(d).
By this, we mean utilizing only the visual modality (V), and only invariance to transformations, which is standard in all recent self-supervised methods~\citep{Wu_2018_CVPR,he2019momentum,Chen2020ASF}.
For this, we consider GDTs of the type $T = (i, m_, \tau, r, g)$ described above and set $K_i=512$ (the largest we can fit in our setup).
In row (a), we pick only the video modality ($m=m_v$ so $K_m=1$).
We also sample a single shift $\tau$ (so $K_\tau=1$), which results in data augmentation but does not learn shift invariance, and no time reversal $r=1$ (so $K_r=1$) --- these are denoted with a $\cdot$ in the table.
However, we do sample two visual augmentations $g$ ($K_g=2$), emulating SimCLR and learning invariance to that factor.
We also set all transformation components to invariance ($c(t_m,t'_m)=1$) except the first that does sample selection.
In row (b-d) we also experiment with adding invariance to \emph{time shift} (TS) and \emph{time reversal} (TR), by setting $K_\tau=2$ and $K_r=2$.
We find that doing so consistently degrades the finetuning accuracy performance, but increases the retrieval performance somewhat, indicating that the model is not able to fully leverage these augmentation signals in a meaningful way.

\paragraph{Cross-modal learning.}

Next, in rows (e-h) we repeat this experiment, but using both audio-visual modalities (AV) by setting $K_m=2$.
In this case, as explained above, we set the weight $w$ to only consider cross-modal interactions and set $K_g=1$.
We note two facts:
First, the performance increases substantially (+7.8\% (e) vs (a-d)).
Second, now TS and TR invariance leads to significant improvements (up to +7.5\%).

\paragraph{Invariance vs distinctiveness.}

Next, in rows (i-l), we explore the effect of being invariant or distinctive to individual transformations, which is unique to our method.
Comparing row (h) to rows (k) and (l), we see that switching to distinctiveness to one of TS or TR further improves performance (up to +1.5\%).
On the other hand, `ignoring' either ($\cdot$ symbols in lines (g) and (j)) is worse than learning invariance ((h) and (l)), with a difference of around 2.5\%.
Finally, in row (m) we find that being distinctive for both TS and TR at the same time is worse, suggesting that a mix of distinctiveness and invariance is preferable.
This is particularly true for the retrieval metric (column R@1).

\subsection{Textual modality}\input{tabs/queryd}

In~\cref{t:text}, we demonstrate the generality of our approach by using ASR captions as an alternative modality (instead of audio) for the HowTo100M dataset~\citep{miech2019howto100m}. 
For the text encoder, we use a simple Word2Vec~\citep{mikolov2013efficient} embedding with a MLP (further details are provided in the Appendix).
Comparing \cref{t:text}(a) with (b), we find that switching from SimCLR to a cross-modal baseline increases performance by more than +22\%. 
Furthermore, we find gains of 3.7\% when switching from data-sampling distinctiveness only (row (b)) to incorporating further distinctivenesses (rows c-d). 
Finally, we find that -- as in the video-audio case -- combining time-shift distinctiveness with time-reversal invariance leads to particularly strong representations (row (f)), yielding benefits of over +5\% compared to data-sampling distinctiveness alone.
Compared to video-audio learning (\cref{tab:hypotheses_results}(m)), we find the case of distinctive-only for video-text learning (\cref{tab:queryd-av}(g)) to be highly competitive, highlighting the need to explore the set of possible transformation signals to achieve the best downstream performance.

\paragraph{Intuition.}
While we only analyse a subset of possible transformations for video data, we nevertheless find consistent signals across both video-audio and video-text learning:
Inclusion of further distinctivenesses to TS and TR always improve upon the basecase and the best setup is achieved for TS distinctiveness and TR invariance. 
One explanation for this might be that there is useful signal in both of these transformations that are not captured by previous ``augmentation-only'' naive noise-contrastive formulations.
For example, for time-shift (TS), the model profits from having to differentiate different points in time, e.g. between an athlete running vs an athlete landing in a sandpit, which could be both in the same video. 
This intuitively serves as a hard negative for the model, increasing its discriminative power.
For time reversal (TR), many actions depicted such as moving an object are inherently invariant to reversing time, as shown in~\citep{price2019_RetroActionsLearning}, therefore yielding a gain when used as an augmentation. 
In ~\citep{wei2018learning}, they show that humans have a $20\%$ error-rate when classifying a video’s direction of time in Kinetics-400, thus demonstrating that Kinetics-400 has subsets of videos that look realistic even when reversed. 
These findings that additional distinctiveness combined with invariances improve video representation learning are noteworthy, as they contradict results from the image self-supervised learning domain, where learning pretext-invariance can lead to more transferable representations~\citep{misra2019selfsupervised}. 
Even when compared to previous self-supervised learning approaches for video-data, such as predicting the arrow of time~\citep{wei2018learning}, our method yields new insights by showing that a unique combination of distinctivenesses and invariances performs best, at least on the training sets considered.
Combining these points, the strong performance of GDT is founded in its ability to leverage highly informative, yet ``free'', signal that we have from construction of the transformations.

\subsection{Qualitative analysis}
\begin{figure}[!t]
    \centering
\includegraphics[width=0.9\columnwidth]{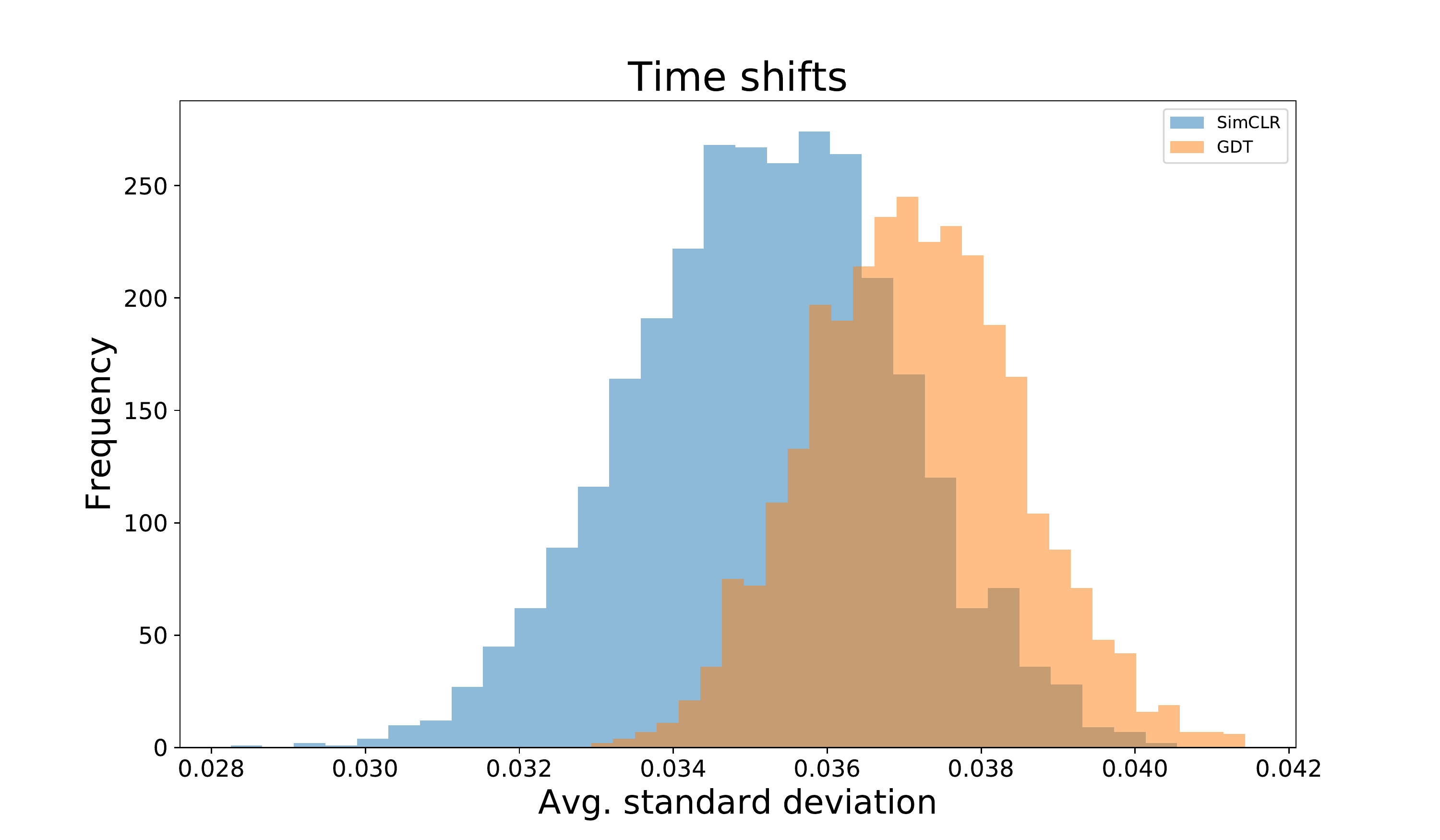}
    \caption{Learning distinctiveness to time-shifts: our GDT model from Tab.1(j) is able to differentiate features from the same video at different times better than a simple SimCLR variant (Tab.1(a)).
    \label{fig:histogram}}
\end{figure}
Here, we study what effect the different transformations we let our model be invariant and distinctive to have on our learned representations.
For this, we compare against the SimCLR baseline of Tab.1(a) and compare the average standard deviation of the normalized features for 10 time-shifted clips per video for 3000 randomly selected Kinetics-400 validation set.

\subsection{Comparison to the state of the art}

\input{tabs/t45_ret-few-aud}

\input{tabs/t3_video-sota}

Given one of our best learning setups from~\cref{s:exp.hypothesis} (row (l)), we train for longer and compare our feature representations to the state of the art on standard downstream benchmarks. 

\subsubsection{Downstream benchmarks}

For \textbf{few-shot classification}, as shown in ~\cref{tab:few-shot-ret}, we significantly beat the 3D-Rotnet~\citep{jing2018self} baseline on UCF-101 by more than $10\%$ on average for each shot with our Kinetics-400 pretrained model.

For \textbf{video retrieval}, we report recall at 1 and 5  retrieved samples for split-1 of the HMDB-51 and UCF-101 datasets in~\cref{tab:few-shot-ret}.
Using our model trained on Kinetics-400, \methodname~significantly beats all other self-supervised methods. 
In particular, we outperform CoCLR~\citep{Han2020SelfsupervisedCF}, a recent state-of-the-art self-supervised method, that uses optical flow as another view to mine hard positive to improve instance discrimination learning for video representations. 
Moreover, we surpass SeLaVi, an audio-visual clustering and representation learning method, by $2\%$ and $10\%$ on average on recall at 1 and 5 for HMDB-51 and UCF-101.  

For \textbf{video action recognition}, we finetune our \methodname~pretrained network for UCF-101 and HMDB-51 video classification, and compare against state-of-the-art self-supervised methods in~\cref{tab:sota}.
When pretrained on the Kinetics datasets, we find that our \methodname~pretrained model achieves very good results, outperforming all recent methods. 
In particular, we outperform audio-visual pretraining methods, AVTS~\citep{avts}, SeLaVi~\citep{asano2020labelling} and XDC~\citep{alwassel2019self}, by large margins using the same architecture (R(2+1)D-18) and dataset (Kinetics-400), showing the effectiveness of our GDT pre-training approach.
We also surpass AVID~\cite{morgado2020avid}, the state-of-the-art audio-visual representation learning method, by $1.5\%$ on HMDB-51 and $3.8\%$ on UCF-101. 
AVID uses a variant of the pre-training scheme of our baseline approach that extends noise contrastive learning to the audio-visual domain as in \Cref{tab:hypotheses_results}, row (e). 
However, while AVID simply encodes sample distinctiveness and invariance to modality in its visual representations, we are able to encode invariances and distinctiveness to additional transformations, which significantly improves our performance. 
Our approach is also more sample efficient, as we are able to achieve our results with 300 less epochs of training. 
Finally, when pretrained on HT100M, we achieve strong gains of $+6.4\%$ on HMDB-51 and $+2.8\%$ on UCF-101 compared to the state-of-the-art video text method, MIL-NCE~\cite{miech2019endtoend}. 
Similar to AVID, MIL-NCE uses a variant of the baseline cross-modal contrastive framework to learn representations, while we are able to improve upon this baseline by learning invariance and distinctiveness to additional transformations such as time reversal and time shift. 
Moreover, with HT100M pre-training, we outperform the Kinetics supervised baseline using the same architecture when finetuned on HMDB-51 ($67.4$ vs $65.1$) and are on par for UCF-101 ($94.1$ vs $94.2$). 
We further show the scalability and flexibility of our GDT framework by pretraining on the IG65M dataset~\citep{Ghadiyaram2019}.
With this, our visual feature representation sets a new state of the art among all self-supervised methods, particularly by a margin of $> \!4\%$ on the HMDB-51 dataset.
On UCF-101, we set similar state-of-the-art performance with XDC.
Along with XDC, we beat the Kinetics supervised pretraining baseline using the same architecture and finetuning protocol.

    


%% file: tabs/t2_hypothesis-res.tex
\newcommand{\vv}{{\bf{d}}}%
\newcommand{\ii}{{\bf{i}}}%
\newcommand{\cc}{{$\cdot$}}%
\begin{table}
\centering
    \caption{\textbf{Learning hypothesis ablation.} Results on action classification performance on HMDB-51 is shown for finetuning accuracy (Acc) and frozen retrieval (recall@1) after pretraining on Kinetics-400 for 50 epochs.
    GDT can leverage signals from both \ul{i}nvariance and stronger \ul{d}istinctiveness transformation signals.
    We consider data-sampling (DS), time-reversal (TR) and time-shifting (TS).
    \label{tab:hypotheses_results}}
      \begin{tabular}{c cccc @{\hskip 5pt} cc }
      \toprule
         & DS & TR & TS & Mod. & $\mathbf{Acc.}$ & $\mathbf{R@1}$ \\
        \midrule
        \multicolumn{7}{l}{\textit{SimCLR-like: DS-distinctiveness only}}      \\
        (a) & \vv & \cc     &  \cc   &   V    & $44.6$ & $11.8$    \\ 
        (b) & \vv & \ii     &  \cc   &   V    & $36.9$ & $13.3$   \\ 
        (c) & \vv & \cc     &  \ii   &   V    & $35.9$ & $15.3$    \\ 
        (d) & \vv & \ii     &  \ii   &   V    & $37.8$ & $13.9$    \\ 
        \midrule
        \midrule
        \multicolumn{7}{l}{\textit{Cross-modal}}\\
        (e) & \vv & \cc & \cc  & AV  & $52.4$ & $21.8$   \\
        (f) & \vv & \ii & \cc  & AV  & $58.8$ & $22.6$   \\
        (g) & \vv & \cc & \ii  & AV  & $57.4$ & $23.5$    \\
        (h) & \vv & \ii & \ii  & AV  & $59.9$ & $24.8$    \\ 
        \midrule
        \multicolumn{7}{l}{\textit{Cross-modal +1 distinctive factor}} \\
        (i) & \vv & \vv & \cc  & AV  & $57.8$ & $26.1$ \\
        (j) & \vv & \cc & \vv  & AV  & $58.7$ & $22.1$  \\
        (k) & \vv & \vv & \ii  & AV  & $61.1$ & $25.4$ \\
        (l) & \vv & \ii & \vv  & AV &  $\bf{61.4}$ & $\bf{27.1}$ \\
        \hline
        \multicolumn{7}{l}{\textit{Cross-modal + 2 distinctive factors}} \\
        (m) & \vv & \vv & \vv  & AV & $57.2$ & $20.5$  \\ %
        \bottomrule
    \end{tabular}
\end{table}

%% file: tabs/queryd.tex
\renewcommand\thesubtable{\Alph{subtable}}

\begin{table}[t] 
\centering
\caption{\textbf{GDT on video-text HT100M dataset}. We also find the positive effect of including more modalities and find non-trivial combinations of beneficial transformations previously unexplored.}\label{t:text}
	{\input{tabs/_HT100GDT}}
\end{table}


%% file: tabs/_HT100GDT.tex
\centering
    \label{tab:queryd-av}
      \begin{tabular}{ccccc  c}
      \toprule
        & DS & TR & TS & Mod. & \textbf{Acc} \\ 
        \midrule
        \multicolumn{6}{l}{\textit{SimCLR-like}} \\
        (a) & \vv & \cc & \cc  & V & $36.1$ \\ 
        \midrule
         \multicolumn{6}{l}{\textit{Video-text cross-modal}} \\
        (b) & \vv & \cc & \cc  & VT  & $59.2$ \\ 
        (c) & \vv & \vv & \cc  & VT  & $61.5$  \\ 
        (d) & \vv & \cc & \vv  & VT  & $62.9$  \\ 
        (e) & \vv & \vv & \ii  & VT  & $63.8$    \\ 
        (f) & \vv & \ii & \vv  & VT  & $\bf{64.4}$    \\ 
        (g) & \vv & \vv & \vv  & VT  & $\bf{64.4}$    \\ 
        \bottomrule
    \end{tabular}

%% file: tabs/t45_ret-few-aud.tex
\begin{table}[t]
\begin{minipage}[t]{0.5\textwidth}
\centering
\input{tabs/_retrieval_and_fewshot}

\end{minipage}
\hfill
\end{table}

%% file: tabs/_retrieval_and_fewshot.tex
\centering
\midsepremove
\caption{\textbf{Video retrieval and Few Shot Learning.} Retrieval accuracy in (\%) via nearest neighbors at various levels of neighborhood sizes and few shot learning accuracy (\%) via  a $k$-nearest neighbor on frozen representations. \label{tab:few-shot-ret} }
\vspace{1mm}
\begin{tabular}{c l c c @{\hskip 7pt} c c}
\toprule
& & \multicolumn{2}{c}{\textbf{HMDB}} & \multicolumn{2}{c}{\textbf{UCF}} \\
& & 1 & 5 & 1 & 5   \\
\midrule
\parbox[t]{2mm}{\multirow{3}{*}{\rotatebox[origin=c]{90}{\textit{kNN}}}}  
&\hphantom{x} Random  & 3.0  & 3.5  & 2.3  & 4.6     \\
&\hphantom{x}  3DRot~\citep{jing2018self} & -- & --   & 15.0 & 31.5  \\
\cmidrule{2-6}
& \hphantom{x} \textbf{GDT (ours)} & \bf{14.3} & \bf{15.4}  & \bf{26.7} & \bf{44.6}  \\
\midrule
&\hphantom{x} SP-Net~\citep{benaim2020speednet}                 & - & - & 13.0 & 28.1   \\
&\hphantom{x} VCP~\citep{luo2020video}                   & 7.6 & 24.4 & 18.6 & 33.6 \\
&\hphantom{x} M-DPC~\citep{Han_2020}                             & 7.7 & 25.7 & 20.2 & 40.4 \\
\parbox[t]{2mm}{\multirow{3}{*}{\rotatebox[origin=c]{90}{\textit{Retrieval}}}}&\hphantom{x} 
VSP~\citep{cho2020selfsupervised}                   & 10.3 & 26.6 & 24.6 & 41.9 \\
&\hphantom{x} CoCLR~\citep{Han2020SelfsupervisedCF}               & 23.2 & 43.2 & 53.3 & 69.4   \\
&\hphantom{x} SeLaVi~\citep{asano2020labelling}                   & 24.8 & 47.6 & 52.0 & 68.6    \\
\cmidrule{2-6}
&\hphantom{x} \textbf{GDT (ours)} & \textbf{26.1} & \textbf{51.7} & \textbf{62.8} & \textbf{79.0}  \\
\bottomrule
\end{tabular}
\midsepdefault

%% file: tabs/t3_video-sota.tex
\begin{table}[t]
	\caption{\textbf{State-of-the-art on video action recognition with full-finetuning.} Self- and fully-supervisedly trained methods on UCF-101 and HMDB-51 benchmarks.
	\labelswitch{tab:sota}
	}
    \include{tabs/full-finetune-joined}
\end{table}

%% file: tabs/full-finetune-joined.tex
\centering
\begin{tabular}{l l @{\hskip 5pt} c c }
		\toprule
		\textbf{Method}                              & \bf{Data}  & 
		\multicolumn{2}{c}{\textbf{Top-1 Acc\%}} \\
		                                             &                           & HMDB\,        & UCF          \\
		\midrule 
      {\color{gray}Supervised~\citep{xie2018rethinking}} & {\color{gray}K-400+IN} & {\color{gray}75.9} & {\color{gray} 96.8} \\
		{\color{gray}Supervised~\citep{alwassel2019self}}    & {\color{gray}K-400}              & {\color{gray}65.1}          & {\color{gray}94.2}         \\   
		\specialrule{.4pt}{1pt}{1pt} 
		AoT~\citep{wei2018learning}                  & K-400              & -             & 79.4         \\  
		MultiSensory~\citep{owens2018audio}          & K-400              & -             & 82.1         \\
		SeLaVi~\citep{asano2020labelling}            & K-400              & 47.1          & 84.2         \\  
		PEMT~\citep{lee2021parameter}                & K-400             & --             & 85.2         \\ 
		XDC~\citep{alwassel2019self}                 & K-400              & 52.6          & 86.8         \\  
		AV Sync+RotNet~\citep{xiao2020audiovisual}   & K-400              & 54.6          & 87.0         \\  
		CoCLR~\citep{Han2020SelfsupervisedCF}        & K-400              & 54.6         & 87.9\\
		SeCo~\citep{yao2021seco}                     & K-400              & 55.6          & 88.3         \\  
		AVTS~\citep{avts}                            & K-400              & 56.9          & 85.8         \\  
		CPD~\citep{li2020learning}                   & K-400              & 57.7          & 88.7        \\  
		AVID~\citep{morgado2020avid}                 & K-400              & 60.8          & 87.5         \\  
		CM-ACC~\citep{ma2021active}                  & K-400              &  61.8         & 90.2         \\
		GLCM~\citep{ma2021contrastive}               & K-400              &  61.9         & 91.2 \\
		\bf{\methodname~(ours)}                      & K-400              & \bf{62.3}    & \bf{90.9}     \\
		 \specialrule{.4pt}{1pt}{1pt}
		 MIL-NCE~\citep{miech2019endtoend}         & HT100M                 & 61.0          & 91.3         \\
		 \bf{\methodname~(ours)}                    & HT100M            & \bf{67.4}     & \bf{94.1}    \\ 
		 \specialrule{.4pt}{1pt}{1pt}
		  XDC~\citep{alwassel2019self}                & IG65M                     & 68.9          & \bf{95.5}         \\  
		 \bf{\methodname~(ours)}                     & IG65M                     & \bf{72.8}     & \ul{95.2}    \\   
		\bottomrule       
	\end{tabular}

%% file: 5_conclusion.tex
\section{Conclusion}

We introduced the framework of Generalized Data Transformations (GDTs), which allows one to capture, in a single noise-contrastive objective, cues used in several prior contrastive and non-contrastive learning formulations, as well as easily incorporate new ones.
The framework shows how new meaningful compositions of transformations can be obtained, encoding valuable invariance and distinctiveness that we want our representations to learn.
Following this methodology, we achieved state-of-the-art results for self-supervised pretraining on standard downstream video action recognition benchmarks, even surpassing supervised pretraining.
Overall, our method significantly increases the expressiveness of contrastive learning for self-supervision, making it a flexible tool for many multi-modal settings, where a large pool of transformations exist and an optimal combination is sought. 



%% file: 6_appendix.tex
\appendix
\counterwithin{figure}{section}
\counterwithin{table}{section}
\section{Appendix}

\subsection{Foundations of compositional contrastive learning}\label{s:theory}

In this section, we develop more formally a basic theory of compositional contrastive learning formulation, providing rigorous grounds for the approach described in~\cref{s:method-new}.

Consider the problem of learning a function $f : \mathcal{X}\rightarrow \mathcal{Y}$.
In a contrastive setting, we are not given information about the values of $f$; instead, we are given a \emph{contrast function}
$
c: \mathcal{X}\times \mathcal{X}\rightarrow \{0,1\}
$
which only tells for which pairs of points $x_1$ and $x_2$ $f$ is the same and for which it differs:

\begin{definition}\label{d:compat}
  The function $f$ is \emph{compatible} with the contrast $c$ if, and only if, for all $x_1,x_2\in\mathcal{X}$:
  $$
  c(x_1,x_2)=\delta_{f(x_1)=f(x_2)}.
  $$
\end{definition}

A contrast function cannot be arbitrary:

\begin{lemma}
The predicate $c(x_1,x_2)=1$ is an equivalence relation if, and only if, there exists a function $f$ compatible with $c$.
\end{lemma}

\begin{proof}
If $c(x_1,x_2)=1$ defines an equivalence relation on $\mathcal{X}$,
then such a function is given by the projection on the quotient $\hat f : \mathcal{X} \rightarrow \mathcal{X}/c = \mathcal{Y}$.
On the other hand, setting $c(x_1,x_2)= \delta_{f(x_1)=f(x_2)}=1$ for any given function $f$ is reflexive, symmetric and transitive because the equality $f(x_1)=f(x_2)$ is.
\end{proof}

\begin{definition}\label{d:admiss}
  The contrast function $c$ is admissible if, and only if, $c(x_1,x_2)=1$ defines an equivalence relation.
\end{definition}

Full knowledge of the contrast function $c$ only specifies the level sets of the function $f$:

\begin{lemma}
Let $f$ be any function compatible with the admissible contrast $c$.
Then, we can write $f=\iota \circ \hat f$ as the composition of an injection
$
\iota : \mathcal{X}/f \rightarrow \mathcal{Y}
$
and the (unique) projection
$
\hat f: \mathcal{X} \rightarrow \mathcal{X}/c
$
of $\mathcal{X}$ onto the equivalence classes $\mathcal{X}/c$ of the equivalence relation $c(x_1,x_2)=1$.
\end{lemma}

\begin{proof}
From elementary algebra, we can decompose any function $f : \mathcal{X}\rightarrow\mathcal{Y}$ as
$$
   f : \mathcal{X}
   \overset{\hat f}{\longrightarrow}  \mathcal{X}/f
   \overset{\iota }{\longrightarrow}  \mathcal{Y}
$$
where $\iota$ is an injective function and $\hat f$ projects $\mathcal{X}$ to the quotient $\mathcal{X}/f$, i.e.~the collection of subsets $X \subset \mathcal{X}$ where $f(x)$ is constant (level sets).
The latter are also the equivalence classes of the relation $f(x_1)=f(x_2)$.
Due to~\cref{d:compat}, this is the same equivalence relation given by the contrast $c$, so that $\mathcal{X}/f = \mathcal{X}/c$.
\end{proof}

Note that, in our contrastive learning formulation, we do \emph{not} define the contrast $c$ on the sample space $\mathcal{X}$, but rather on the transformation space $\mathcal{T}$.
The following lemma suggests that defining a contrast $c(T,T')$ on transformations instead of data samples is usually acceptable:

\begin{lemma}
Let $c:\mathcal{T}\times\mathcal{T}\rightarrow \{0,1\}$ be an admissible contrast function defined on a set (e.g., a batch) of generalized data transformations $\mathcal{T}$.
Furthermore, let $x$ be a dataset and let $x(T)\in\mathcal{X}$ be the sample indexed by transformation $T$.
If $x(T) = x(T') \Rightarrow T = T'$ (i.e.~different transformations output different samples), then setting $\tilde c(x(T), x(T'))=c(T,T')$ defines part of an admissible sample contrast function $\tilde c$ on $\mathcal{X}$.
\end{lemma}

\begin{proof}
The expression above defines $\tilde c$ on the sample space
$
\tilde{\mathcal{X}}=\{x(T): T\in\mathcal{T}\} \subset \mathcal{X}.
$
Reflexivity, symmetry and transitivity are inherited from $c$.
However, if the same data point $x(T)=x(T')$ can be obtained from two different transformations $T$ and $T'$, the definition is ill posed.
The hypothesis in the lemma guarantees that this is not the case.
\end{proof}

\subsubsection{Compositional transformations}\label{s:composition}

Next, we consider the case in which $T = (t_1,\dots,t_M)$ is a composition of individual transformations $t_m$, each with its own contrast $t_m$:

\begin{definition}
  We say that a contrast function $c(t_m,{t'}_m)$ is distinctive if it is given by $\delta_{t_m={t'}_m}$. We say that it is invariant if it is identically one.
\end{definition}

The following lemma provides a formula for the overall contrast function $c(T,T)$ given the contrasts for the individual factors.

\begin{lemma}\label{l:prod}
Let $c(t_m,t_m')=1$ be admissible contrast functions, either distinctive or invariant.
Then, the product $c(T,T') = \prod_{m=1}^Mc(t_m,t_m')$ is also admissible.
\end{lemma}

\begin{proof}
The reflexive and symmetric properties are obviously inherited.
For the transitive property, note that $c(T,T') = 1$ if, and only if, $\forall m: c(t_m,t_m')=1$.
Hence:
\begin{multline*}
   c(T,T') =c(T',T'') = 1 \\
   ~~~\Rightarrow~~~
   \forall m : c(t_m,t_m')=c(t_m',t_m'')=1
   \\
   ~~~\Rightarrow~~~
   \forall m : c(t_m,t_m'')=1
   ~~~\Rightarrow~~~
  c(T,T'')=1.
\end{multline*}
\end{proof}

Finally, we show that, essentially, the formula above is the only reasonable one.
For this, we only require $c(T,T')$ to be monotonic in the individual factors; i.e., if more factors become $1$, then $c(T,T')$ can only grow:

\begin{definition}
We say that $c(T,T')$ is monotonic in the individual factors if, and only if, for any three transformations $T,T',T''$ such that
$c(t_m,{t'}_m) \leq c(t_m,t''_m)$
for all the factors, then we also have 
$c(T,T')\leq c(T,T'')$.
\end{definition}

Next, we show that $c$ can only have a very limited form:

\begin{lemma}\label{l:prod2}
Suppose that the admissible monotonic contrast $c(T,T')$ is expressible solely as a function of the individual admissible contrasts $c(t_m,{t'}_m)$ for $m=1,\dots,M$.
Then, up to a permutation of the transformations, we can always write
$$
 c(T,T') = \prod_{i=1}^{m} c(t_i,{t'}_i)
$$
where $0\leq m \leq M$.
In particular, $m=M$ is the only option that is guaranteed not to ignore some of the factors.
\end{lemma}

\begin{proof}
From the assumptions, we can write
$$
c(T,T') = h \circ v(T,T')
$$
where $h$ is a function of the binary vector
$$
 v(T,T') = (c(t_1,{t'}_1), \dots, c(t_M,{t'}_M)) \in \mathbb{B}^M.
$$
Furthermore, since invariant factors are constant, they do not affect the function; hence, without loss of generality we can assume that all factors are distinctive.

Since all factors are distinctive, we can construct two transformations
$T'=({t'}_1,\dots,{t'}_M)$
and
$T''=(t''_1,\dots,t''_M)$
such that $v(T',T'')=(0,\dots,0)$ (i.e., all the contrasts $c({t'}_m,t''_m)$ are null).
If $c(T',T'') = 1$, then, due to monotonicity, $c(T',T'')$ is identically 1 and the lemma is proved for $m=0$.

If not, let $c(T',T'')=h(0,\dots,0)=0$. Then, for any given binary vector $v$, we can construct a transformation $T=(t_1,\dots,t_M)$ such that $v(T,T')=v$ and $v(T,T'')=\bar v$ as follows:
$$
   t_m = \begin{cases}
    t_m', & \text{if\ } v_m = 1,\\
    t_m'', & \text{otherwise}.
   \end{cases}
$$
We cannot have $c(T,T')=h(v) = h(\bar v) = c(T,T'') =1$; otherwise, due to the transitivity of $c$, we would have $c(T',T'')=h(0,\dots,0)=1$, which contradicts our assumption.
Hence, $h$ must partition the space of binary vectors in two halves, the ones for which $h(v) = 1$ and their complements $h(\bar v)=0$.

Now let $v$ be a vector with the minimal number of 1 such that $h(v)=1$. Again without  loss of generality, we can assume this is of the type $v=(1,\dots,1,0,\dots,0)$ with $m$ ones in front.
Due to monotonicity, all vectors of type $v'=(1,\dots,1,v_{m+1},\dots,v_M)$ must also have $h(v')=1$; by taking their complement, the previous result shows that all vectors $v''=(0,\dots,0,v_{m+1},\dots,v_M)$ must have $h(v'')=0$. This is also the case for any vector of the type $(v_1,\dots,v_{m},0,\dots,0)$ where any $v_i=0$ for $1\leq i\leq m$ (because $m$ is the minimum number of ones required for $h(v)=1$).
We conclude that $h(v)=1$ if, and only if, $(v_1,\dots,v_{m})=(1,\dots,1)$.
\end{proof}

\subsubsection{Forming batches}\label{s:batches}

Let $\hat{\mathcal{T}}_1\times\cdots\times\hat{\mathcal{T}}_M$ be a composite space of generalized data transformations, so that data points are indexed as $x(t_1,\dots,t_M)$.
Furthermore, let $c(t_m,t_m')$ be corresponding admissible contrast functions and let $c(T,T')$ be their product, as in~\cref{l:prod}.
As before, we assume that the functions are of two kinds:
\begin{itemize}
  \item invariant: $c(t_m,t_m')=1$.
  \item distinctive: $c(t_m,t_m')=\delta_{t_m=t_m'}$.
\end{itemize}
Let $I \subset \{1,\dots,M\}$ be the subset of indices $m$ corresponding to the invariant transformations and $D = \{1,\dots,M\}\backslash I$ the distinctive ones.

Let 
$
\operatorname{sample}(\hat{\mathcal{T}}_m;K_m)
$
be a stochastic operator that samples $K_m \leq |\hat{\mathcal{T}}_m|$ transformations from $\hat{\mathcal{T}}_m$ without replacement.
We sample a batch recursively:
\begin{itemize}
  \item Let $\mathcal{T}_1 = \operatorname{sample}(\hat{\mathcal{T}}_1;K_1)$
  \item Let 
  $
  \mathcal{T}_{m} = 
  \bigcup_{T\in \mathcal{T}_{m-1}} T \circ \operatorname{sample}(\hat{\mathcal{T}}_m;K_m)$
\end{itemize}
At each level of the recursion, each transformation is extended by sampling $K_m$ more transformations (note that no two identical transformations can be generated in this manner).
Hence $|\mathcal{T}_M| = K_1\cdots K_M$.

\begin{lemma}\label{l:count1}
There are exactly $(\prod_m K_m)(\prod_{m \in I} K_m)$ pairs of transformations $(T,T')\in\mathcal{T}_M\times\mathcal{T}_M$ for which $c(T,T')=1$.
Of these, exactly $\prod_m K_m$ are trivial pairs ($T=T'$).
Hence, there are $(\prod_m K_m)(\prod_{m \in I} K_m - 1)$ non-trivial pairs for which $c(T,T')=1$.
\end{lemma}

\begin{lemma}\label{l:count2}
For each $T \in \mathcal{T}_M$, there are exactly $(\prod_m K_m) - (\prod_{m \in I} K_m)$ pairs $(T,T'')$ such that $c(T,T'')=0$.
\end{lemma}

For example, in SimCLR $M=2$, $D=\{1\}$, $I=\{2\}$, 
$K_1=B/2$, $K_2=2$, $|\mathcal{T}_2|=B$, there are
$B(2-1)=B$ non-trivial pairs of transformations for which $c(T,T')=1$, and, for each $T$, there are $B - 2$ pairs of transformations for which $c(T,T'')=0$.

The lemmas above suggest that we should pick $K_m\geq 2$ for at least one invariant factor and at least $K_m\geq 2$ for at least one distinctive factor, as otherwise~\cref{e:contrastive} is degenerate.

\subsubsection{Limitations}\label{s:limitations}

In general, we want more restrictive requirements than the one described above.
When learning $f$, difficult (and therefore interesting) cases amount to: learning to be sensitive to `small' variations in the distinctive factors and learn to be insensitive to `large' variations in the invariant factors.
For the former, we would like $f$ to observe variations in a single distinctive factor at a time, as these are the `smallest'.
For these individual variations to exist at all in the batch, we should choose $K_m \geq 2$ for all distinctive factors $m\in D$.

Even so, the hierarchical scheme in general \emph{prevents} us from observing all individual variations.
In fact, suppose that two transformations $T$ and $T'$ in $\mathcal{T}_M$ differ for factor $m$ (i.e. $t_m\not=t_m'$).
Then, the remaining factors $t_{m+1}\not=t_{m+1}',\dots$ also differ in general because successive transformations are sampled independently in different branches of the tree.
This means that we cannot, in general, observe a change in $t_{m}$ \emph{alone}, so the function $f$ may not learn to be distinctive to this `minimal' change in isolation.

Note that this is a limitation that affects our sampling scheme as well as existing methods such as SimCLR.
Fortunately, in practice this is often not an issue.
There are in fact two mitigating factors, which apply to most existing formulations, including the new ones presented here.

First, some transformations spaces $\hat{\mathcal{T}}_m$ are very small, and in fact binary (e.g., modality splitting, time reversal).
In this case, $K_m=2$ means that transformations are sampled exhaustively, so for level $m$ the hierarchical sampling scheme does extract all possible combinations of transformations.

Second, in other cases the issue is moot due to the nature of the transformations and the data.
For instance, in SimCLR the first transformation $t_1$ amounts to sampling a certain image $x_i$, and the second transformation $t_2$ amounts to sampling two data augmentation $g_{1i}(x_i)$ and $g_{2i}(x_i)$, different for each image.
The issue here is that we cannot observe a change in the index $i$ for the same augmentation $g(x_1)$ and $g(x_2)$, as these data points do not exist in the batch.
This means that the representation $f$ can only learn to distinguish two different images $x_i$ that also have two different augmentations applied to them.
Because of the particular nature of the training data (ImageNet) this is likely irrelevant since different images $x_i$ are unrelated in any case, so applying transformations does not significantly alter their statistical relationships.

However, note that this is not \emph{always} the case.
For instance, if SimCLR was applied to a dataset of pre-aligned faces (for the purpose of learning face recognition), then being unable to contrast different faces $g(x_1)$ and $g(x_2)$ under the same transformation $g$ would make negative pairs far to easy to discriminate.

\subsection{Reduction in variance theorem}

\subsubsection{Proof of \cref{thm:variance}}
For ease of notation, we will express eq. \ref{e:contrastive} as the expected value of a loss function $\ell$, which subsumes the weight ($w$), contrast ($c$), feature extractor $(\Phi$) and log-softmax functions:

\begin{equation}
\mathcal{L}=\mathbb{E}_{T,T'\sim\hat{\mathcal{T}}}\left[\ell\left(x(T),x(T')\right)\right].\label{eq:ideal-contrastive-objective}
\end{equation}

The expectation is over pairs of transformations in $\hat{\mathcal{T}}=\hat{\mathcal{T}}_{1}\times\ldots\times\hat{\mathcal{T}}_{M}$, the space of all compositions of transformations, which can be applied to the data $x$. Note that eq. \ref{e:contrastive} contains a sum over a third transformation ($T''$) to compute the softmax's normalization, which is also subsumed by $\ell$ in eq. \ref{eq:ideal-contrastive-objective} as this third transformation is not essential for the rest of the proof. We will separate each transformation into invariant and distinctive parts, $T=(T^{I},T^{V})$ respectively with $T^{I}\in\hat{\mathcal{T}}_{I}$ and $T^{V}\in\hat{\mathcal{T}}_{V}$ (see sec. \ref{s:batches}). Note that this separation is merely a notational convenience; the individual transformations can be applied to the data in \emph{any} order, with $x(T^{I},T^{V})=x(t_{1}\circ\ldots\circ t_{M})$, and each $t_{i}$ belonging to either $T^{I}$ or $T^{V}$. Then, eq. \ref{eq:ideal-contrastive-objective} becomes:
\[
\mathcal{L}=\mathbb{E}_{T^{I},T'^{I}\sim\hat{\mathcal{T}}_{I},\,T^{V},T'^{V}\sim\hat{\mathcal{T}}_{V}}\left[\ell\left(x(T^{I},T^{V}),x'(T'^{I},T'^{V})\right)\right].
\]

Consider a mini-batch of data sample pairs and their associated transformation compositions, $\mathcal{B}_{\textrm{direct}}=\left\{ T_{i}^{I},T_{i}^{V},{T'}_{i}^{I},{T'}_{i}^{V}\right\} _{i=1}^{K_{I}^{2}K_{V}^{2}}$, sampled as $T_{i}^{I},{T'}_{i}^{I}\sim\hat{\mathcal{T}}_{I}$ and $\,T_{i}^{V},{T'}_{i}^{V}\sim\hat{\mathcal{T}}_{V}$. The batch size is a function of $K_{I}=\prod_{j\in I}K_{j}$ and $K_{V}=\prod_{j\in V}K_{j}$, the number of sampled invariant and distinctive transformations in our method, respectively. The batch size of $K_{I}^{2}K_{V}^{2}$ was chosen to allow a direct comparison. The expected value of the loss over this batch is then the simple empirical average:
\begin{equation}
\hat{\mathcal{L}}_{\textrm{d}}=\frac{1}{K_{I}^{2}K_{V}^{2}}\sum_{i}^{K_{I}^{2}K_{V}^{2}}\ell\left(x(T_{i}^{I},T_{i}^{V}),x({T'}_{i}^{I},{T'}_{i}^{V})\right).\label{eq:direct-sampling-objective}
\end{equation}

Now consider the domain of transformed samples $\mathcal{X}=\{x(T^{I},T^{V}):T^{I}\in\hat{\mathcal{T}}_{I},T^{V}\in\hat{\mathcal{T}}_{V}\}$. Due to the assumed injectivity of all $t\in T^{I}$, we may partition the domain using one partition $\mathcal{X}_{j}=\{x(T^{I},T_{j}^{V}):T^{I}\in\hat{\mathcal{T}}_{I}\}$ per distinctive transformation $T_{j}^{V}$, i.e.: $\mathcal{X}=\cup_{j}^{K_{V}}\mathcal{X}_{j}$, with $\mathcal{X}_{j}\cap\mathcal{X}_{j'}=\emptyset,\,\forall j,j'$. The probability distribution of the samples has density $p(T^{I},T^{V})$, and the density in each partition is thus $p_{j}(T^{I})=K_{V}p(T^{I})\delta_{T^{I}\in\mathcal{X}_{j}}$, with $\mathbb{\delta}$ the indicator function.

GDT can then be interpreted as a stratified sampling method, with one stratum (partition) per \emph{pair} of distinctive transformations. The domain being sampled by the expectation in eq. \ref{eq:ideal-contrastive-objective} is $\mathcal{X}^{2}=\cup_{jj'}^{K_{V},K_{V}}\mathcal{X}_{j}\times\mathcal{X}_{j'}$, and stratified sampling consists of sampling an equal number of $K_{I}^{2}$ sample pairs from each of the $K_{V}^{2}$ partitions:
\begin{equation}
\hat{\mathcal{L}}=\frac{1}{K_{V}^{2}K_{I}^{2}}\sum_{ii'jj'}^{K_{I},K_{I},K_{V},K_{V}}\ell\left(x(T_{i}^{I},T_{j}^{V}),x(T_{i'}^{I},T_{j'}^{V})\right).\label{eq:gdt-sampling-objective}
\end{equation}
Note the subtle difference from eq. \ref{eq:direct-sampling-objective} in the summation ranges, and that the \emph{same} samples and transformations are reused for both elements of each pair, instead of being sampled independently to fill a mini-batch. To make the following derivations easier, note that we can equivalently express eq. \ref{eq:gdt-sampling-objective} as:
\[
\hat{\mathcal{L}}=\frac{1}{K_{V}^{2}}\sum_{jj'}^{K_{V},K_{V}}\hat{\mathcal{L}}_{jj'},
\]
 with $\hat{\mathcal{L}}_{jj'}=\frac{1}{K_{I}^{2}}\sum_{ii'}^{K_{I},K_{I}}\ell\left(x(T_{i}^{I},T_{j}^{V}),x(T_{i'}^{I},T_{j'}^{V})\right)$. We will first show that this pairwise stratified sampling is an unbiased estimate of eq. \ref{eq:ideal-contrastive-objective}:
\begin{alignat*}{1}
\mathbb{E}[\hat{\mathcal{L}}] & =\frac{1}{K_{V}^{2}}\sum_{jj'}^{K_{V},K_{V}}\mathbb{E}[\hat{\mathcal{L}}_{jj'}]\\
 & =\frac{1}{K_{V}^{2}}\sum_{jj'}^{K_{V},K_{V}}\mathcal{L}_{jj'}\\
 & =\mathcal{L},
\end{alignat*}
where we use the expectation $\mathcal{L}_{jj'}$ of the loss function evaluated on the partition $jj'$ (corresponding to distinctive transformations with indices $j$ and $j'$), as $\mathcal{L}_{jj'}=\mathbb{E}_{T^{I}\in\mathcal{X}_{j},T'^{I}\in\mathcal{X}_{j'}}\left[\ell\left(x(T^{I},T_{j}^{V}),x(T'^{I},T_{j'}^{V})\right)\right]$.

Similarly, we can also define each partition's loss variance $\sigma_{jj'}^{2}=\mathbb{V}_{T^{I}\in\mathcal{X}_{j},T'^{I}\in\mathcal{X}_{j'}}\left[\ell\left(x(T^{I},T_{j}^{V}),x(T'^{I},T_{j'}^{V})\right)\right]$. Then, from eq. \ref{eq:gdt-sampling-objective} we obtain directly
\begin{alignat*}{1}
\mathbb{V}[\hat{\mathcal{L}}] & =\frac{1}{K_{V}^{4}}\sum_{jj'}^{K_{V},K_{V}}\mathbb{V}[\mathcal{L}_{jj'}]\\
 & =\frac{1}{K_{V}^{4}K_{I}^{2}}\sum_{jj'}^{K_{V},K_{V}}\sigma_{jj'}^{2}.
\end{alignat*}

As a point of comparison, the variance of the direct sampling estimate is:
\begingroup \allowdisplaybreaks
\begin{alignat*}{1}
\mathbb{V}[\hat{\mathcal{L}}_{\textrm{d}}] & =\frac{1}{K_{V}^{2}K_{I}^{2}}\left((\mathbb{E}_{(T^{I},T'^{I})\in\mathcal{X}^{2}}[\ell^{2}(x(T^{I},T^{V}),x(T'^{I},T'^{V})\right.\\
 & \qquad\qquad\;\;\left.\vphantom{\mathbb{E}_{(T^{I},T'^{I})\in\mathcal{X}^{2}}x(T'^{I},T'^{V})}x(T'^{I},T'^{V}))]-\mathcal{L}^{2})\right)\\
 & =\frac{1}{K_{V}^{2}K_{I}^{2}}\left(\frac{1}{K_{V}^{2}}\sum_{jj'}^{K_{V},K_{V}}\mathbb{E}_{T^{I}\in\mathcal{X}_{j},T'^{I}\in\mathcal{X}_{j'}}\right.\\
 & \qquad\qquad\;\;\left.\vphantom{\frac{1}{K_{V}^{2}}\sum_{jj'}^{K_{V},K_{V}}}\left[\ell^{2}\left(x(T^{I},T_{j}^{V}),x(T'^{I},T{}_{j'}^{V})\right)\right]-\mathcal{L}^{2}\right)\\
 & =\frac{1}{K_{V}^{2}K_{I}^{2}}\left(\frac{1}{K_{V}^{2}}\sum_{jj'}^{K_{V},K_{V}}\left(\mathcal{L}_{jj'}^{2}+\sigma_{jj'}^{2}\right)-\mathcal{L}^{2}\right)\\
 & =\frac{1}{K_{V}^{4}K_{I}^{2}}\sum_{jj'}^{K_{V},K_{V}}\left(\sigma_{jj'}^{2}+\left(\mathcal{L}_{jj'}-\mathcal{L}\right)^{2}\right)\\
 & \geq\frac{1}{K_{V}^{4}K_{I}^{2}}\sum_{jj'}^{K_{V},K_{V}}\sigma_{jj'}^{2}
\end{alignat*}

\endgroup

completing the proof.\qed

\subsection{Additional experimental results}\label{s:appx:more}

\subsubsection{Modality ablation\label{s:appx:cross-vs-within-modal}}

In \cref{tab:appx:ablation:modality}, we provide the results of running our baseline model (sample-distinctiveness only) within-modally instead of across modalities and find a sharp drop in performance.
\vspace{-1em}
\begin{table}[htb]
\setlength{\tabcolsep}{3pt}
\caption{\textbf{Within vs cross-modal learning.} Results on action classification performance on HMDB-51 and UCF-101 is shown for finetuning accuracy (Acc) and frozen retrieval (recall@1) after pretraining on Kinetics-400 for 50 epochs.}
  \centering
  \footnotesize
  \begin{tabular}{l c c c c @{\hskip 5pt} c c c c}
  \toprule
      & \multicolumn{2}{c}{\textbf{HMDB}} & \multicolumn{2}{c}{\textbf{UCF}} \\
                   & \textbf{Acc.} & \textbf{R@1}                 & \textbf{Acc.} & \textbf{R@1}  \\
    \midrule 
          Within-modal & $37.8$ & $13.9$           & $76.4$ & $28.0$ \\
          Cross-modal  & $\textbf{52.4}$ & $\textbf{21.8}$ & $\textbf{87.6}$ & $\textbf{54.8}$  \\
    \bottomrule
  \end{tabular}
  \vspace{5pt}
  \label{tab:appx:ablation:modality}
\end{table}

\subsubsection{Dataset details \label{s:appx:datasetdetails}}
The Kinetics-400 dataset~\citep{kinetics} is human action video dataset, consisting of 240k training videos, with each video representing one of 400 action classes. After filtering out videos without audio, we are left with 230k training videos, which we use for pretraining our model.

HT100M~\citep{miech2019howto100m}  is a large-scale instructional video collection of 1.2 million Youtube videos, along with automatic speech recognition transcripts. 
There are more than 100 million clips (ASR segments) defined in HowTo100M.




HMDB-51~\citep{kuehne2011hmdb} consists of 7K video clips spanning 51 different human activities.
HMDB-51 has three train/test splits of size 5k/2k respectively.

UCF-101~\citep{UCF101} contains 13K videos from 101 human action classes, and has three train/test splits of size 11k/2k respectively.

IG65M~\citep{Ghadiyaram2019} is a large-scale weakly supervised dataset collected from a social
media website, consisting of 65M videos of human action events. We use the all the videos in the
dataset for pretraining.



\subsubsection{Preprocessing details \label{s:appx:preprocdetails}}
The video inputs are 30 consecutive frames from a randomly chosen starting point in the video.
These frames are resized such that the shorter side is between 128 and 160, and a center crop of size 112 is extracted, with color-jittering applied.
A random horizontal flip is then applied with probability 0.5, and then the inputs' channels are z-normalized using mean and standard deviation statistics calculated across each dataset.

One second of audio is processed as a $1\times40\times99$ image, by taking the log-mel bank features with 40 filters and 99 time-frames after random volume jittering between 90\% and 110\% is applied to raw waveform, similar to~\citep{Arandjelovic17}.
The spectrogram is then Z-normalized, as in~\citep{avts}.
Spec-Augment is then used to apply random frequency masking to the spectrogram with maximal blocking width 3 and sampled 1 times. Similarly, time-masking is applied with maximum width 6 and sampled 1 times.

For the text, we remove stop words from the narrations as in ~\citep{miech2019howto100m}. 
For each narration, we take a maximum of 16 consecutive words covering a max duration of 4 seconds as in ~\citep{miech2019endtoend}. 

\subsubsection{Pretraining details \label{s:appx:traindetails}}
We use R(2+1)D-18~\citep{tran2018closer} as the visual encoder $f_v$ and ResNet~\citep{KaimingHe16} with 9 layers as the audio encoder $f_a$ unless otherwise noted; both encoders produce a fixed-dimensional output ($512$-D) after global spatio-temporal average pooling.
For the text encoder, we use the Google News self-supervised pre-trained word2vec (d=300) embedding~\citep{mikolov2013efficient}, that is linearly projected to 2048D and max-pooled as in ~\citep{miech2019endtoend}.
After the inputs are encoded by their respective modality encoders, the vectors are then passed through two fully-connected layers with intermediate size of $512$ to produce $256$-D embeddings as in~\citep{bachman2019learning} which are normalized by their L2-norm~\citep{Wu_2018_CVPR}.
The embedding is used for computing the contrastive loss, while for downstream tasks, a  linear layer after the global spatio-temporal average pooling is randomly initialized.
For NCE contrastive learning, the temperature $\rho$ is set as $1/0.07$.
For optimizing these networks, we use SGD.
The SGD weight decay is $10^{-5}$ and the SGD momentum is $0.9$.
We use a mini-batch size of $8$ on each of our $64$ GPUs giving an effective batch size of $512$ for distributed training.
The initial learning rate is set to $0.01$ which we linearly scale with the number of GPUs, after following a gradual warm-up schedule for the first $10$ epochs~\citep{goyal2017accurate}.
For Kinetics, we train for $100$ epochs (3 days), while for HT100M, we train for $40$ epochs (3 days).

\subsubsection{Ablation experiment details}
For the ablations, we only pretrain for $50$ epochs on the Kinetics-400 dataset, and $20$ epochs on the HT100M dataset, since it is a much larger dataset. 

For downstream evaluation, we only evaluate on the first fold of HMDB-51 but found the performance between folds to be close (within 1\%).



\subsubsection{Evaluation details}
All evaluation code is provided in the Supplementary Material.

\paragraph{Video Action Recognition.}
During training, we take 10 random clips of length 32 frames from each video.
For video clip augmentations, we follow a standard protocol as in \citep{avts}.
During evaluation, we uniformly sample 10 clips from each video, average softmax scores, and predict the class having the highest mean softmax score.
We then measure the mean video top-1 accuracy across all videos and all official folds.
During training, we use SGD with initial learning rate $0.0025$, which we gradually warm up to $2\cdot 10^{-2}$ in the first $2$ epochs. The weight decay is set to $5\cdot 10^{-3}$ and momentum to $0.9$.
We use a mini-batch size of $32$ and train for $12$ epochs with the learning rate multiplied by $5\cdot 10^{-2}$ at $6$ and $10$ epochs.
We compare our GDT pretrained model with both self-supervised methods, and supervised pretraining, and report average top-1 accuracies on UCF101 and HMDB-51 action recognition task across three folds in~\cref{tab:sota}.

\paragraph{Few-shot classification}
We follow the protocol in \citep{jing2018self} and evaluate our our \methodname~pretrained network using few-shot classification on the UCF-101 dataset, and additionally on HMDB-51.
We randomly sample $n$ videos per class from the train set, average the encoder's global average pooled features from ten clips per training sample and measure classification accuracy performance on the validation set using a $k$-nearest neighbor classifier, with $k$ set to $1$.

\paragraph{Video Retrieval.}
We follow the standard protocol as outlined in~\citep{clip_order}.
We use the split 1 of UCF101, and additionally HMDB-51.
We uniformly sample 10 clips per video, and average the max-pooled features after the last residual block for each clip per video.
We use these averaged features from the validation set to query the videos in the training set.
The cosine distance of representations between the query clip and all clips in the training set are computed.
When the class of a test clip appears in the classes of $k$ nearest training
clips, it is considered to be correctly predicted.
We report accuracies for $k = 1, 5, 20$ and compare with other self-supervised methods on UCF101 and HMDB-51 in ~\cref{tab:few-shot-ret}.

\begin{figure}[h]
\center
\includegraphics[width=1\linewidth]{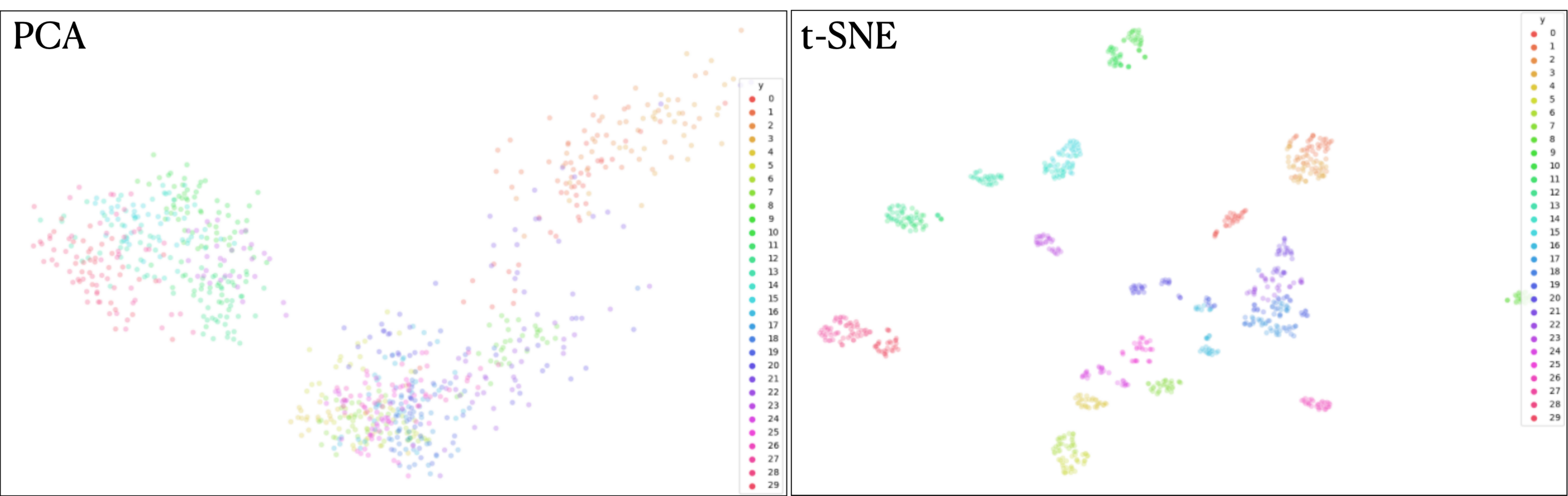}
\caption{Feature visualizations with PCA and t-SNE on 30 videos of a single, random class of HMDB-51. For each video, we sample 10 temporal clips and encode video-ID with color. Embeddings are generated from our time-shift distinct model (Tab.1 (l)).
}
\label{fig:tsne}
\end{figure}

\subsubsection{Additional Qualitative analysis}
In~\cref{fig:tsne}, we present a PCA and t-SNE~\citep{van2008visualizing} plots of the features obtained by our model (DS-d, TR-d, TS-d) (Tab. 1, row (l)).
We observe that even comparing to videos of the same action category, the individual clips are well separated, showing that the model is learning to distinguish between different time intervals.